\documentclass{article} 
\usepackage{iclr2027_conference,times}

\usepackage{amsmath,amsfonts,bm}

\def\eqref#1{equation~\ref{#1}}

\def\1{\bm{1}}

\DeclareMathAlphabet{\mathsfit}{\encodingdefault}{\sfdefault}{m}{sl}
\SetMathAlphabet{\mathsfit}{bold}{\encodingdefault}{\sfdefault}{bx}{n}

\iclrfinalcopy

\usepackage{hyperref}
\usepackage{url}

\usepackage{algorithm}
\usepackage{algpseudocode}
\usepackage{amsmath}
\usepackage{amsthm}
\usepackage[normalem]{ulem}
\usepackage{subcaption}
\usepackage{graphicx}
\usepackage{booktabs}
\usepackage{multirow}
\usepackage{tikz}
\hypersetup{hidelinks}
\usetikzlibrary{positioning, fit, arrows.meta, backgrounds}
\newtheorem{proposition}{Proposition}

\newif\ifdraft
\drafttrue 
\definecolor{ForestGreen}{RGB}{34,139,34}

\ifdraft
 \newcommand{\PF}[1]{{\color{red}{\bf PF: #1}}}
 
 \newcommand{\OH}[1]{{\color{blue}{\bf OH: #1}}}
 
 \newcommand{\ohrm}[1]{{\color{blue}\sout{#1}}}
 \newcommand{\ED}[1]{{\color{ForestGreen}{\bf ED: #1}}}
 
 \newcommand{\edrm}[1]{{\color{ForestGreen}\sout{#1}}}
\else
 \newcommand{\PF}[1]{}
 
 \newcommand{\OH}[1]{}
 
 \newcommand{\ohrm}[1]{}
 \newcommand{\ED}[1]{}
 
 \newcommand{\edrm}[1]{}
\fi

\newcommand{\parag}[1]{\vspace{-3mm}\paragraph{#1}}

\newcommand{\ours}[0]{{\it SPARROW}}

\title{Generative Refinement for Low-Budget Black Box Optimization}

\author{%
  Edouard R. Dufour \qquad
  Oliver Heinimann \qquad
  Pascal Fua\\
  CVLab, EPFL\\
  Lausanne, Switzerland\\
  \texttt{\{edouard.dufour,oliver.heinimann,pascal.fua\}@epfl.ch}
}

\begin{document}

\maketitle


\begin{abstract}

Black-box optimization is a fundamental tool in science and engineering for optimizing objectives when gradient information is unavailable. It becomes especially difficult when the objective function is expensive to evaluate, limiting the evaluation budget to a few tens or hundreds of queries, and when good solutions occupy complex, low-measure regions of the search space. Generative models can supply useful structural priors in such settings, but existing generative BBO approaches bring significant evaluation cost. We identify three design principles for generative optimization under such low-budget conditions: avoid objective learning, optimize in candidate space, and make every evaluation count. Together, these principles motivate separating structural modeling from objective-driven search. We instantiate them in \ours{}, a simple sequential optimizer that maintains a persistent, ranked archive of evaluated candidates and uses a fixed, unconditional generative sampler solely as a corruption–refinement operator. \ours{} requires only access to the sampler's corruption and refinement processes, and never needs to evaluate the objective to train or guide it. Across three complementary settings, probing thin feasible geometry, disconnected high-performing regions, and failure-prone evaluations, \ours{} outperforms classical and generative baselines under strict evaluation budgets. These results demonstrate that separating structural priors from objective-driven search can be effective when evaluations are scarce and the search geometry is challenging.

\end{abstract}


\section{Introduction}

Many problems in science and engineering require optimizing without access to gradient information, as in materials design \citep{Frazier15a}, drug discovery \citep{Gomez18a,Gruver23a}, and engineering simulation \citep{Diessner22a,Wu21a}. This is referred to as {\it Black-Box Optimization} (BBO). It is particularly challenging when evaluations require expensive simulation or physical experimentation, as in high-fidelity engineering and experimental design. In industrial practice, this limits the number of evaluations that can be performed to tens or hundreds, rather than tens of thousands. The difficulty is further magnified when the evaluation feedback is noisy or failure-prone and when good solutions lie on thin, curved or disconnected regions of the search space. We refer to this regime as the very-low-budget (VLB) setting and focus on these scenarios in this paper. In this regime, only limited objective information is available to guide optimization.

Classical approaches can struggle in the VLB-BBO regime: Bayesian optimization (BO) is efficient when a reliable surrogate model is available~\citep{Shahriari16a}, but sparse data relative to objective complexity can misdirect acquisition and waste evaluations. Evolutionary strategies (ES) such as CMA-ES~\citep{Hansen16} sidestep this surrogate requirement, but can require more evaluations to adapt their search distribution and may waste evaluations when that distribution is poorly matched to the geometry of good solutions.

In recent years,  diffusion \citep{Ho20a,Song21b,Song21c} and flow matching \citep{Lipman22} models have emerged as powerful tools for representing complex, high-dimensional distributions. This has motivated generative approaches to BBO when domain data is available~\citep{Krishnamoorthy23a, Wu24c, Uehara25a}. Such data may provide rich structural information while remaining unlabeled with respect to the objective of interest. Unfortunately, these methods require either a large number of evaluations of the objective function on the training dataset, or at inference time, both of which are incompatible with the VLB setting.

In this paper, we first identify three design principles for generative VLB-BBO. In this scenario, we propose to: avoid objective learning, perform optimization in candidate space, and make every evaluation count. Taken together, these principles suggest a different way to use generative models in BBO, that is, as fixed structural priors, decoupled from the objective signal. We then introduce \ours{} (\textbf{S}equential \textbf{P}roposal via \textbf{A}rchival \textbf{R}ank-based \textbf{R}efinement for \textbf{O}ptimization under \textbf{W}eak feedback) as a simple instantiation of these principles. It keeps an unconditional generative sampler fixed, performs objective-driven operations directly over fully denoised candidates, and optimizes sequentially over a persistent ranked archive, allowing every evaluation to inform subsequent proposals.

We demonstrate strong empirical performance on complementary problems that stress existing optimization approaches: low-measure feasible subspaces, disconnected high-performing regions, and unreliable evaluations. \ours{} consistently outperforms classical and generative baselines under strict evaluation budgets. Our key contribution is therefore a new way to use generative models as fixed structural priors, decoupled from the objective signal as motivated by the design principles, for efficient BBO in challenging scenarios.


\section{Related Work}

\subsection{Classical black-box optimization}

Bayesian optimization (BO) repeatedly fits a surrogate of the objective and selects new candidates through an acquisition function~\citep{Shahriari16a}. Local methods such as TuRBO~\citep{Eriksson19a} improve scalability by restricting optimization to adaptive trust regions. BO can be highly sample-efficient when the surrogate is informative, but sparse observations relative to objective complexity can make fitting such surrogate difficult. Evolutionary strategies avoid this by guiding the search directly through evaluated candidates. Among them, CMA-ES \citep{Hansen16} is one of the most popular and operates by sampling a Gaussian distribution whose parameters are adapted based on previous evaluations. While often effective, it can become inefficient when the distribution of high-performing candidates is not easily captured by a Gaussian. Differential evolution~\citep{Lampinen04} instead estimates search directions from differences between evaluated candidates, but generally requires larger populations and evaluation budgets for sustained progress. For a more extensive description of classical BBO, we refer the interested reader to~\cite{Conn09a}.

\subsection{Generative black-box optimization}

Generative models have recently been used in BBO to provide structure over complex candidate spaces. Existing methods mainly differ in where the objective enters the generative process: during training, at inference time, or externally through evolutionary selection. These choices lead to different trade-offs in sample efficiency and robustness.

\paragraph{Objective-dependent generative optimization.}

Many approaches condition the generative model on the objective to sample high-performing candidates directly. CbAS \citep{Brookes19a}, DDOM \citep{Krishnamoorthy23a}, Diffusion-BBO \citep{Wu24c}, BONET \citep{Mashkaria23a}, and DiBO \citep{Yun25a} learn objective-dependent generation from labeled data or online evaluations, while DEMO \citep{Yuan24a} and RGD \citep{Chen24h} rely on learned surrogates. Under very small budgets, the available objective evaluations are often insufficient to reliably fit these objective-dependent models, and resulting model errors can bias subsequent proposals.

\paragraph{Inference-time guidance.}
Other approaches keep the generator fixed and use objective information directly during sampling. DNO optimizes diffusion noise \citep{Tang24b}, while DAS and RERD bias denoising trajectories toward high rewards \citep{Kim25b, Uehara25a}. This optimization is indirect, as updates are transformed by the remaining generative trajectory before producing a candidate. Prior work reports that this can degrade guidance \citep{Tan25a,Tang24b}. Sampling-based methods can additionally require several reward computations during a single generation. Both issues are undesirable in a VLB setting.

\paragraph{Generative evolutionary search.}
Generative models have also been used directly as mutation operators in evolutionary search, keeping the optimization signal in candidate space. Diffusion-ES \citep{Yang24d} selects and refines populations of trajectories using an unconditional diffusion model, while EvoDiff-NAS \citep{Tang26a} uses diffusion-guided mutation for neural architecture search. Both were developed for application-specific settings where population-level evaluations are affordable, and therefore do not target the evaluation constraints of VLB-BBO.


\section{Design Principles for Generative VLB-BBO}
\label{sec:principles}

In very-low-budget black-box optimization, objective evaluations are the main scarce resource. This changes how generative priors should be used for optimization. We identify three desirable design principles.

\parag{Avoid objective learning.}
With only tens to hundreds of evaluations affordable, fitting a surrogate, objective-conditioned generator, or reward-dependent target can produce unreliable approximations whose errors directly bias future proposals. This affects both classical BO and objective-dependent generative methods. Thus, in the VLB regime, it is better to use objective evaluations directly for optimization without approximating the objective.

\parag{Optimize in candidate space.}
Many existing methods use objective information as guidance along the generative trajectory. In the VLB regime, this signal is already weak because it is estimated from very few evaluations. Applying it through intermediate noisy states can further distort its effect before a final candidate is produced~\citep{Tan25a,Tang24b}. Thus, objective-driven operation should be performed in candidate space, leaving the generative trajectory unguided.

\parag{Make every evaluation count.}
When the total budget is small, performing multiple evaluations internally and discarding them after use is wasteful. For example, SMC-based guidance uses many internal evaluations for one single pass, and existing generative evolutionary approaches discard past evaluations when updating the current population. We therefore favor sequential optimization that retains every evaluation and lets it inform future proposals.

Together, these principles suggest a clean separation of roles. The generative model provides structural bias, while objective information drives optimization only in candidate space. 

\section{Method}
\label{sec:method}

\ours{} instantiates the three design principles. At each iteration, it selects a parent from the archive, applies a rank-guided mutation, and uses the fixed sampler to corrupt and refine the resulting candidate.

\subsection{Background}
\label{sec:background}
    We briefly review two of the building blocks used by \ours{}: generative corruption--refinement operators, and rank-based comparisons of objective feedback.
    \paragraph{Generative Samplers and Refinement via Partial Noising}
        We consider generative models that map noise to data through a trajectory. Let $t \in [0,1]$ denote a noise level, with $t=0$ being pure noise and $t=1$ the data distribution, and let $\mathcal{X}_t$ denote the space of states at noise level $t$. 
        We take
        \[S_{t \rightarrow 1} : \mathcal{X}_t \rightarrow \mathcal{X}\]
        to be the refinement operator that maps a partially noisy state to a sample in data space by following the generative trajectory. Sampling corresponds to applying $S_{0 \rightarrow 1}$ to a draw from the noise prior, inducing a distribution $\nu$ over $\mathcal{X}$.

        Let
        \[\mathcal{N}_t : \mathcal{X} \rightarrow \mathcal{X}_t\]
        denote the corruption operator at noise level $t$, which maps a candidate $x \in \mathcal{X}$ to a noisy state $x_t \sim \mathcal{N}_t(x)$ according to the forward process used during the training of the sampler. The composition of $\mathcal{N}_t$ and $S_{t\rightarrow 1}$ induces a family of stochastic transition kernels over the search space, indexed by a controllable corruption level
        \[\mathcal{T}_t(x) := S_{t \rightarrow 1}(x_t), \quad \text{where } x_t \sim \mathcal{N}_t(x).\]
        This construction generalizes stochastic editing procedures such as SDEdit \citep{Meng22}, where partial corruption followed by denoising biases inputs towards the sampler's data distribution. Here, $\mathcal{N}_t$ and $S_{t \rightarrow 1}$ are defined abstractly. In practice, any generative sampler supporting a consistent corruption--refinement decomposition can be used, without access to its internal parametrization.

        The noise level $t$ controls the exploration scale: large $t$ produces local refinements, small $t$ induces global modifications. In particular, $\mathcal{T}_0(x)\sim\nu$ for all $x\in \mathcal{X}$, since full corruption erases all input structure and produces samples from the prior.

    \paragraph{Rank-based Optimization and Invariance}
        Rank-based optimization replaces raw objective values with their induced ordering, so that all decisions depend on $\operatorname{rank}(x)$ rather than $f(x)$. This yields invariance to strictly monotone transformations of the objective \citep{Ollivier17a, Hansen16}, making the optimization signal more robust to scaling, misspecification, and outcome-dependent noise. The tradeoff is that ranking discards magnitude information, making rank-based updates less efficient than value-based ones when reliable magnitude information is available. In our regime, however, this tradeoff is favorable: robustness and invariance outweigh potential efficiency when evaluations are scarce.

\begin{figure}[ht]
    \includegraphics[width=\textwidth]{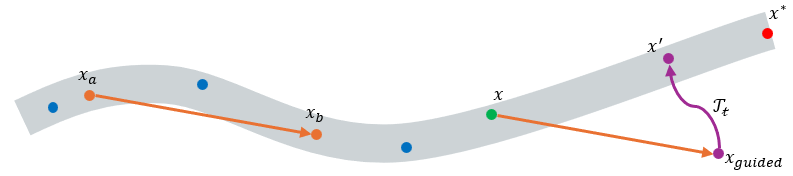}
    \caption{One iteration of \ours{} on a thin tube problem, with optimality on one end of the subspace. Archive points lie on the sampler support. A rank-guided directional step moves the parent in a promising direction. The operator $\mathcal{T}_t$ then biases the candidate back towards the sampler support.}
    \label{fig:alg_sketch}
\end{figure}

\begin{algorithm}
    \caption{\ours{}}
    \label{alg:method}
    \begin{algorithmic}[1]
        \Require Initial archive $\mathcal{A}_0 = \{(x_i,f(x_i))\}_{i=1}^{n_0}$,
        fixed corruption--refinement operator $\mathcal{T}_t$, budget $B$,
        selection pressure $\beta$, corruption exponent $\gamma$, step size $\lambda$

        \For{$k = 0,1,\dots,B-n_0-1$}

            \State Compute normalized archive ranks \Comment{Lower rank is better}
            \[r_k(x_i) = \frac{\operatorname{rank}_{\mathcal{A}_k}(x_i)}{n_k}\]

            \State Sample parent \Comment{Prefer better candidates. Keep non-zero probability on all}
            \[x_k \sim p_{\mathrm{par}}(\cdot\mid\mathcal{A}_k) \propto \exp(-\beta r_k(\cdot))\]
            
            \State Set rank-dependent corruption level \Comment{\begin{tabular}[t]{r}
                Good parents: local refinement.\\
                Poor parents: stronger corruption
            \end{tabular}}
            \[t_k \gets 1-r_k(x_k)^\gamma\]
            
            \State Sample $x_a,x_b\in\mathcal{A}_k$ uniformly at random

            \State Form guided candidate \Comment{Differential evolution step}
            \Statex\hfill\Comment{Larger step when stronger refinement follows}
            \[x_{\mathrm{guided}} \gets x_k + \lambda(1-t_k)\operatorname{sign}\!\big(r_k(x_b)-r_k(x_a)\big)(x_a-x_b)\]

            \State Apply the structural prior $\displaystyle \quad x'_{k+1}\sim\mathcal{T}_{t_k}(x_{\mathrm{guided}})$

            \State Evaluate and add to the archive $\displaystyle \quad \mathcal{A}_{k+1} \gets \mathcal{A}_k\cup\{(x'_{k+1},f(x'_{k+1}))\}$

        \EndFor

        \State \Return best candidate in the final archive
    \end{algorithmic}
\end{algorithm}

\subsection{Algorithm and Design Properties}
    As is described in Algorithm~\ref{alg:method}, \ours{} is a sequential algorithm that iteratively adds single newly evaluated candidates to a persistent archive. At each iteration, as illustrated in Figure~\ref{fig:alg_sketch}, \ours{} first selects a parent from the ranked archive, favoring high-performing candidates while retaining nonzero probability for all archive members. It then samples two additional archive elements and uses their ranks to orient a differential-evolution direction toward the better candidate, providing a second objective-driven search signal. The displaced parent is corrupted to a rank-dependent noise level and refined through the fixed generative sampler. Better-ranked parents are corrupted less and are therefore refined more locally, while lower-ranked parents are corrupted more and induce broader exploration. The resulting proposal is evaluated once and added to the archive. This implementation follows our three design principles:

    \parag{No objective learning.}
    The optimization is driven by two simple mechanisms. First, parent selection favors high-ranking candidates in the archive. We show in Appendix~\ref{app:maxent_parent_selection} that it follows a maximum-entropy distribution under an expected-rank constraint~\citep{Eiben25}, a standard form of rank-based selection in BBO. Second, the directional mutation, inspired by differential evolution~\citep{Lampinen04}, provides an additional stochastic bias toward high-performing regions. Neither mechanism requires learning or approximating the objective.

    \parag{Optimization in candidate space.}
    The optimization engines, parent selection and directional mutation, operate entirely on denoised candidates. The generative sampler is never guided by the objective and only acts afterwards through $\mathcal T_t$, pulling proposals toward regions represented by the structured data distribution. The corruption level controls the strength of this refinement, but no objective information is injected along the generative trajectory.

    \parag{Leveraging every evaluation.}
    \ours{} is entirely sequential and performs no evaluations for internal guidance computations. All information obtained from evaluations is retained in the archive and can influence future proposals. Reducing feedback to ranks further improves robustness in poorly understood objective landscapes.

    Together, these properties decouple objective-driven search from structural modeling. The archive drives optimization, while the fixed generative sampler acts only as a structural refinement operator.

    We show in Appendix~\ref{app:asymptotic_global_consistency} that the proposed search mechanism eventually converges to the global optimum of the sampler's support. Finite-budget performance is evaluated empirically in Section~\ref{sec:results}.


\section{Experiments}
\label{sec:results}
\subsection{Tasks}
We evaluate \ours{} on three complementary scenarios chosen to pose distinct challenges to VLB-BBO. First, a synthetic thin-tube benchmark involves a low-measure structured region, highlights the importance of structured priors and enables controlled ablation of the geometric difficulty. Second, HopperController from Design-Bench~\citep{Trabucco22a} requires optimization in a 5126-dimensional landscape with separated high-performing regions. Third, airfoil aerodynamic optimization is of direct industrial relevance and tests the method given failure-prone simulation-based evaluations.
Additional results on the remaining continuous Design-Bench tasks are reported in Appendix~\ref{app:designbench}. These smoother settings help delineate the regime in which the added structural prior is most useful.

\subsection{Evaluation Protocol}
All methods are compared under the same budget of objective-function calls, including those used internally. Since exhaustive hyperparameter tuning is itself incompatible with the VLB setting, we sequentially test five human-informed configurations per
tunable method and task, using one run per configuration. The selected configuration is then evaluated over
10 independent runs. We report the median and interquartile range of the best objective value as a function of evaluations consumed. More implementation details are given in Appendix \ref{app:exp_implem}, and the tuning procedure is detailed in Appendix \ref{app:tuning}.

\subsection{Baselines}
    We compare \ours{} against three classical BBO baselines and three generative BBO baselines spanning unconditional sampling, reward-guided refinement, and surrogate-guided generation.
    
    CMA-ES~\citep{Hansen16} is the standard evolution strategy, maintaining an adaptive Gaussian search distribution in ambient space. GP is a Gaussian process surrogate with expected improvement acquisition, a strong standard BO baseline. TuRBO~\citep{Eriksson19a} augments GP with local trust regions, often more adapted as dimensionality scales. 
    
    RDS (Random Diffusion Sampling) draws proposals unconditionally from the generative prior without selection or archive, which separates the contribution of guidance from the structural contribution of the prior alone. RERD~\citep{Uehara25a} biases generation trajectories through repeated sampling during inference. We also construct GPGuide, which combines the GP surrogate with the generative prior, doing conditional guidance toward high rewards along the generation trajectory. 

    We exclude methods that train an objective-conditioned generative model from online evaluations because $B=100$ observations are insufficient to fit such models reliably and would make the comparison poorly representative of their intended use.
    

\begin{figure}[ht]
    \centering
    \includegraphics[width=\textwidth]{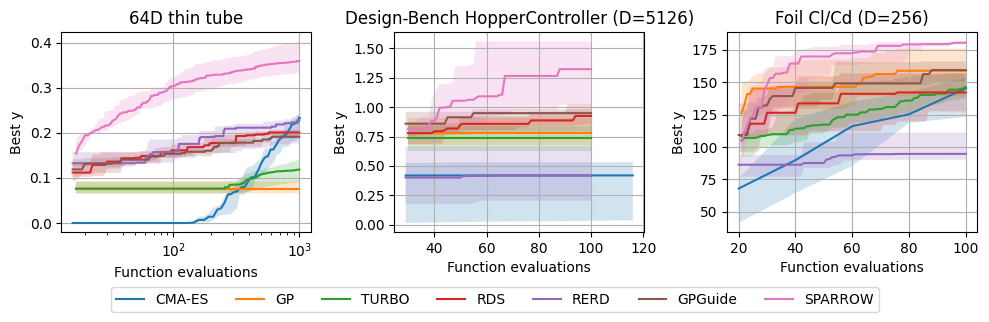}
    \caption{\small {\bf Objective median value and 25/75 Inter Quartile Range (IQR)}  as a function of the number of times the objective function is evaluated, up to $B=1000$ for the 64-D tube on the left and to $B=100$ when maximizing the HopperController reward and airfoil $C_l/C_d$ on the middle and right. }
    \label{fig:main_results}
\end{figure}

\begin{table}[ht]
    \centering
    \small
    \caption{\small Best objective value found (median, IQR over 10 runs). The tube is evaluated at two budgets to assess scaling behavior; airfoil at $B=100$, consistent with realistic aerodynamic design constraints; HopperController at $B=100$ to probe performance in the low-budget regime on a standard benchmark.}
    \label{tab:main_results}
    \begin{tabular}{lcccc}
        \toprule
        & \multicolumn{2}{c}{\textbf{64D Tube}} & \multicolumn{1}{c}{\textbf{HopperController}} & \multicolumn{1}{c}{\textbf{Airfoil $C_l/C_d$}} \\
        \cmidrule(lr){2-3} \cmidrule(lr){4-4} \cmidrule(lr){5-5}
        \textbf{Method} & $B=100$ & $B=1000$ & $B=100$ & $B=100$ \\
        \midrule
        CMA-ES          & $0.00_{[0.00,\, 0.00]}$ & $0.23_{[0.19,\, 0.24]}$ & $0.42_{[0.03,\, 0.53]}$ & $145.30_{[123.26,\, 153.68]}$\\
        GP              & $0.08_{[0.07,\, 0.09]}$ & $0.08_{[0.07,\, 0.09]}$ & $0.78_{[0.68,\, 0.91]}$ & $158.85_{[139.20,\, 175.94]}$\\
        TuRBO           & $0.08_{[0.07,\, 0.09]}$ & $0.12_{[0.09,\, 0.14]}$ & $0.74_{[0.63,\, 0.84]}$ & $146.32_{[137.31,\, 157.08]}$\\
        RDS             & $0.16_{[0.14,\, 0.17]}$ & $0.20_{[0.19,\, 0.21]}$ & $0.93_{[0.85,\, 1.02]}$ & $142.11_{[127.47,\, 156.78]}$\\
        RERD            & $0.16_{[0.13,\, 0.19]}$ & $0.23_{[0.18,\, 0.24]}$ & $0.42_{[0.20,\, 0.67]}$ & $ 94.67_{[ 90.12,\, 111.11]}$\\
        GPGuide         & $0.15_{[0.15,\, 0.16]}$ & $0.19_{[0.19,\, 0.20]}$ & $0.95_{[0.79,\, 0.98]}$ & $159.37_{[150.04,\, 166.12]}$\\
        \ours{} (ours)  & $\mathbf{0.30}_{[0.28,\, 0.32]}$ & $\mathbf{0.36}_{[0.33,\, 0.40]}$ & $\mathbf{1.32}_{[1.02,\, 1.56]}$ & $\mathbf{180.47}_{[179.14,\, 181.54]}$\\
        \bottomrule
    \end{tabular}
\end{table}

\begin{figure}[ht]
    \centering
    \includegraphics[width=\textwidth]{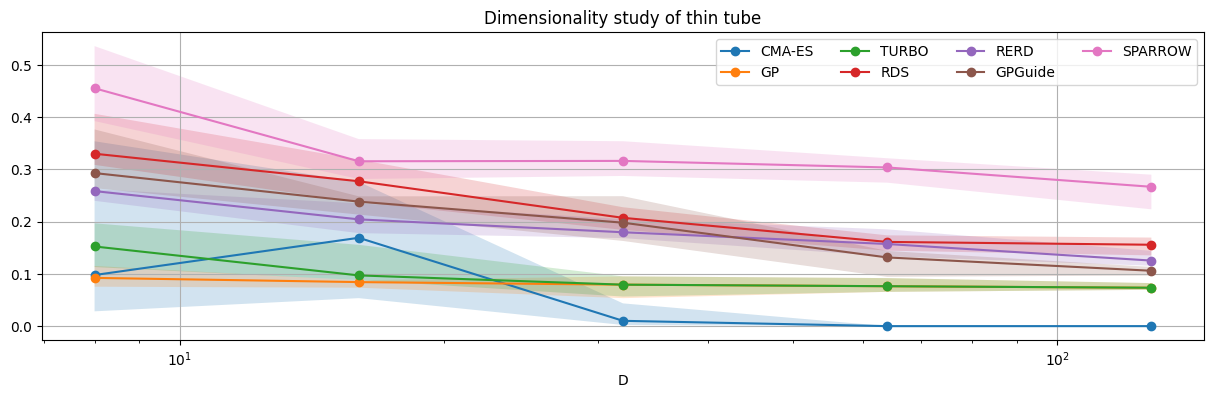}
    \caption{\small {\bf Objective median value and IQR as a function of the tube dimension} using $B=100$ evaluations. \ours{} again outperforms the baselines across dimensions. 
    The gap between \ours{}  and RDS widens with increasing $D$, indicating that rank-based selection becomes
    increasingly valuable as the geometry grows more complex. Exact values are reported in Appendix~\ref{app:ablation}.}
    \label{fig:dim_ablation}
\end{figure}

\subsection{Thin High-Dimensional Tube}
 \label{sec:toy_tube}

    We define a synthetic benchmark in which optimal solutions lie near the end of a thin, curved, one-dimensional manifold embedded in a high-dimensional ambient space. Figure~\ref{fig:alg_sketch} illustrates the geometry. While synthetic, this setting shows the core challenge \ours{} is designed for: most of the search space yields poor or invalid solutions, and progress requires staying near a structured feasible region.

    \parag{Experimental Setup.}
    The benchmark consists of a smooth curved tube of radius $r=0.1$ with a start and finish. The objective rewards points near the finish, but only within the tube. Points outside the tube score worse than points near the start. We report main results at $D=64$, and vary $D \in \{8,16,32,64,128\}$ for the ablation. The sampler is trained on points sampled along the tube with a $\mathrm{Beta}(1,3)$ distribution over arc-length, concentrating mass near the start and away from the optimum. This reflects a realistic scenario where structured data is abundant, but sparse near the optimum. Full formalism and implementation are given in Appendix~\ref{app:tube}.

    \parag{Results.}
    Fig.~\ref{fig:main_results} (left) and Tab.~\ref{tab:main_results} show reward values as a function of consumed evaluation budget. At $B=100$, \ours{} clearly outperforms all baselines and retains the best performance through $B=1000$.
    
    CMA-ES makes no progress in the VLB regime, but begins improving after several hundred evaluations. Its delayed improvement is consistent with the Gaussian search distribution requiring many evaluations before concentrating on the thin feasible region. GP remains close to its initial performance, consistent with most ambient-space evaluations providing little information about progress along the thin tube. TuRBO improves modestly at larger budgets, suggesting that localization helps but does not overcome the limited feedback.

    RDS makes consistent progress, confirming that the generative prior encodes useful geometric structure. GPGuide slightly underperforms RDS. Early in the search, its GP is fitted to very few evaluations, most of which provide little information about progress along the tube. Its guidance is therefore largely extrapolative and can mislead the sampler rather than improve it.
    RERD slightly outperforms RDS, showing that reward-guided refinement helps. However, it spends part of the budget on importance weighting and resampling during refinement, rather than directly accumulating new optimization feedback.

    \ours{}'s guidance adds value beyond the prior. Fig.~\ref{fig:dim_ablation} shows that performance generally degrades as $D$ increases because the tube occupies a smaller fraction of ambient space. Nevertheless, \ours{} outperforms all baselines, with a more pronounced relative advantage over unconditional sampling at high dimensions. This suggests that guidance becomes particularly valuable as dimensionality grows.

\subsection{HopperController}
\label{sec:hopper}

    HopperController is a $D=5126$ continuous control task, in which the search space is the parameter space of a neural network policy. The objective is a combination of the hopper's velocity, control cost, and a bonus for staying upright. The landscape contains many highly disconnected high-performing regions, as shown in Appendix~\ref{app:designbench}.

    \parag{Experimental Setup.}
    We train an unconditional flow matching model on the HopperController training set and initialize all methods from the generative prior. Training points are excluded from initialization because the dataset contains a small number of near-optimal solutions that prevent any method from improving upon the initial archive. The objective is evaluated using the Gaussian process oracle provided by Design-Bench.

    \parag{Results.}
    Fig.~\ref{fig:main_results} (middle) and Tab.~\ref{tab:main_results} show results at $B=100$. \ours{} substantially outperforms all baselines.

    CMA-ES fails to identify an optimization signal, as in the early stages of the tube experiment. Its Gaussian search distribution is poorly matched to a $5126$-dimensional space with a structured feasible region and low feedback. GP and TuRBO plateau near the initialization quality, consistent with their surrogates being too poorly calibrated to identify improvement directions. 
    
    RDS makes modest progress, confirming that the generative prior encodes useful structure. GPGuide improves slightly further, but remains limited by the sparse feedback available to fit its surrogate. RERD stays close to its initial performance. Unlike the connected tube, Hopper contains disconnected high-performing regions, where progressive reward-guided refinement may be less effective. Its repeated reward queries further limit how many refinement attempts fit within $B=100$. \ours{} instead combines local refinement of good archive candidates with stronger corruption, up to near-unconditional resampling, allowing it to explore new regions while retaining all previously discovered solutions.

\subsection{Airfoil Aerodynamic Optimization}
 \label{sec:airfoil}

    We evaluate \ours{} on a real-world aerodynamic design task, maximizing the lift-to-drag ratio $C_l/C_d$ of a two-dimensional airfoil computed by XFOIL~\citep{Drela89}, a widely used panel-method solver. The landscape of aerodynamic objectives is notoriously complex due to turbulence, flow separation, and operation outside the solver's validity range, causing frequent solver failures. As they still consume computational resources, we treat such failures as evaluations: They count against the evaluation budget but are not added to the archive.

    \parag{Experimental Setup.}
    We use the generative prior of DiffAirfoil~\citep{Wei24b}, a diffusion model coupled to an auto-decoder that maps a 256-dimensional latent code to airfoil coordinates. We train it on the UIUC airfoil database~\citep{Selig96a} and generated NACA profiles, providing a prior over geometrically valid airfoils without aerodynamic supervision. \ours{} operates entirely in this latent space, using DiffAirfoil's forward and reverse processes as the mutation operator $\mathcal{T}_t$. All baselines likewise operate in the latent space. The auto-decoder is used only for decoding before evaluation by XFOIL. More details about the implemented pipeline and the behaviour of its latent space are given in Appendix~\ref{app:diffairfoil}.

    \parag{Results.}
    Fig.~\ref{fig:main_results} (right) and Tab.~\ref{tab:main_results} show results at $B=100$. \ours{} improves rapidly and reaches the highest final $C_l/C_d$ with a notably tight IQR.

    GP performs strongly early in the search before slowing down. CMA-ES and TuRBO improve more gradually throughout the budget but remain below the strongest methods. This suggests that the aerodynamic objective is sufficiently capturable from sparse observations for traditional methods to be effective. 

    RDS also makes steady progress, confirming that the generative prior captures useful airfoil structure, while RERD makes little improvement under the limited budget. GPGuide performs similarly to GP. Its surrogate is learned from evaluated latent codes but used to guide partially noised states, where its predictions are less reliable.


\begin{figure}[ht]
    \includegraphics[width=\textwidth]{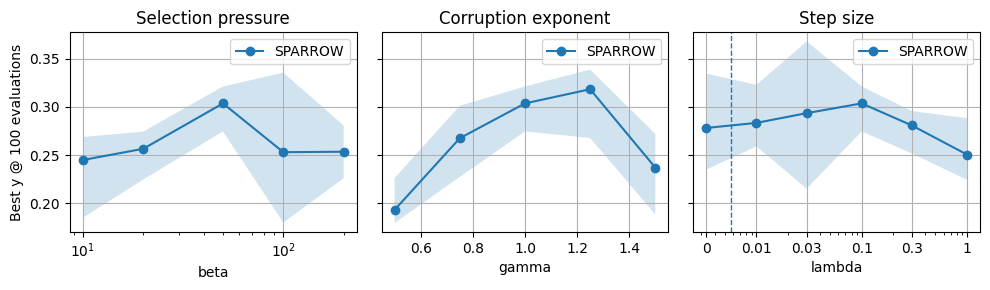}
    \caption{Hyperparameter sensitivity on the 64D thin tube at $B=100$ evaluations (median and IQR over 5 runs). Each parameter is varied independently with others fixed at their default values $\beta=50$, $\gamma=1$, $\lambda=0.1$. Exact values are reported in Appendix~\ref{app:ablation}.}
    \label{fig:abl_params}
\end{figure}
   
\subsection{Hyperparameter Sensitivity}
    \label{sec:hyperparam}

    \ours{} has three hyperparameters: selection pressure $\beta$, corruption exponent $\gamma$, and mutation step size $\lambda$. Fig.~\ref{fig:abl_params} shows performance at $B=100$ as each is varied independently around the defaults $\beta=50$, $\gamma=1$, $\lambda=0.1$ on the 64D tube. Low $\beta$ makes parent selection nearly uniform and rank information ceases to guide the search. At high $\beta$, excessive elitism causes stagnation. At $\gamma=0$, $t_k = 0$ for all candidates, reducing the method to RDS. High $\gamma$ produces near-identity proposals around top-ranked candidates, causing stagnation around current elites. At $\lambda=0$, directional mutation is removed entirely. Performance remains strong but below the best nonzero setting, showing that directional guidance is useful but not essential. At large $\lambda$, aggressive mutations increasingly move candidates away from good regions before refinement.
    
    In all cases \ours{} performs well across a broad intermediate range, though $\gamma$ is the most sensitive parameter. For our experiments, parameters were derived from these qualitative insights rather than exhaustive search, suggesting that the sensitivity analysis provides useful guidance for selecting practical settings. More details are given in Appendix~\ref{app:tuning}.


\section{Limitations}

\ours{} is designed for low-budget optimization on complex geometric landscapes with unreliable feedback. With simple objectives, reliable feedback, or high evaluation budgets, we expect more established methods to be more effective.

\ours{}'s performance is bounded by the sampler quality: Solutions outside the generative prior's support are unreachable. In practice, this means \ours{} inherits any biases of the prior.

The three hyper-parameters $\beta$, $\gamma$, $\lambda$ require domain-specific initialization. Section~\ref{sec:hyperparam} provides a sensitivity analysis that we found sufficient to transfer across our three domains, but more principled adaptation strategies or less sensitive evolutions of \ours{} remain future work.

\section{Conclusion}
We studied generative black-box optimization in the very-low-budget regime, where objective evaluations are the main scarce resource. We argued for three design principles: avoid objective learning, perform optimization in candidate space, and make every evaluation count. Together, they suggest separating structural modeling from objective-driven search.

Guided by these principles, we introduced \ours{}, a simple sequential optimizer combining a persistent ranked archive with a fixed generative refinement operator. Across complementary tasks with thin feasible geometry, disconnected high-performing regions, and failure-prone evaluations, this simple instantiation achieved the strongest final performance under strict evaluation budgets. These results support the broader design principle that generative models can be most useful in VLB-BBO as fixed structural priors rather than objective-dependent optimizers.

\newpage
\subsection*{AI use statement}

We used ChatGPT and Claude as research-support tools, primarily for critical review of manuscript drafts, feedback on the presentation of methodology and experiments, language editing and rephrasing, and literature search and discovery. The research ideas, methods, experimental design, mathematical arguments, and reported results were developed by the authors. AI-assisted code was limited to non-algorithmic tasks such as visualization and formatting and did not affect the implementation logic or generation of experimental results. No generative AI was used to generate experimental data. All AI-assisted material was manually reviewed, and technical claims, numerical results, and references were independently checked. The authors take responsibility for the full content of the paper.

\subsection*{Ethics statement}
This work studies general-purpose black-box optimization and does not involve human subjects, personal data, or sensitive datasets. We are not aware of ethical concerns specific to the methods or experiments beyond those associated with the general use of optimization methods.

\subsection*{Reproducibility Statement}
We provide the information required to reproduce the reported results throughout the paper and appendices. A clean repository will be released upon publication.

\newpage
\bibliography{bib/string, bib/vision, bib/optim, bib/learning, bib/robotics,bib/cfd, bib/biomed, bib/misc}
\bibliographystyle{iclr2027_conference}

\newpage
\appendix

\section{Maximum-entropy derivation of rank-based parent selection}
    \label{app:maxent_parent_selection}
    We show that the exponential rank-based selection rule
    \[p_i = \frac{\exp(-\beta r_i)}{\sum_j \exp(-\beta r_j)}\]
    arises as the maximum-entropy distribution over archive elements under a constraint on expected rank.

    Let \(r_i \in [0,1]\) denote the normalized rank of candidate \(x_i\). We seek a distribution \(p=(p_1,\dots,p_n)\) that maximizes entropy
    \[H(p) = -\sum_i p_i \log p_i\]
    subject to
    \[\sum_i p_i = 1,\qquad\sum_i p_i r_i = \rho.\]

    The Lagrangian is
    \[\mathcal{L}(p,\lambda,\beta)=-\sum_i p_i \log p_i+ \lambda\!\left(\sum_i p_i - 1\right)- \beta\!\left(\sum_i p_i r_i - \rho\right).\]

    Setting \(\partial \mathcal{L}/\partial p_i = 0\) gives
    \[\log p_i = \lambda - 1 - \beta r_i,\]
    hence
    \[p_i \propto \exp(-\beta r_i).\]

    The parameter \(\beta\) controls the expected rank: \(\beta=0\) yields uniform selection, while larger \(\beta\) increases preference for better-ranked candidates.

\section{Asymptotic Global Consistency over the Sampler Support}
    \label{app:asymptotic_global_consistency}
    We show that \ours{} is asymptotically globally consistent over the support of the unconditional sampler. The proof follows the classical framework  of \citet{Solis81}. The key point is that \ours{} naturally  satisfies its conditions: the exponential selection rule assigns positive  probability to the worst-ranked candidate at every iteration, which produces an  unconditional sample from $\nu$, thereby ensuring persistent exploration of the sampler support without requiring explicit random restarts.

    Let $\nu$ denote the sampler's resulting distribution and assume that full corruption produces an unconditional sample, i.e.,
    \[\mathcal{T}_0(x) \sim \nu \quad \text{for all } x \in X.\]
    Let $S = \operatorname{supp}(\nu)$, assume evaluations are noiseless, and let  $M = \sup_{x \in S} f(x)$. Assume that the supremum is approachable under $\nu$, i.e., for all $\delta > 0$,
    \[\nu\!\left(\{x \in S : f(x) > M - \delta\}\right) > 0.\]
    
    \begin{proposition}
        Under the above assumptions,
        \[\max_{0 \le j \le k} f(x_j)\xrightarrow[k\to\infty]{a.s.}M.\]
    \end{proposition}

    \begin{proof}
        At iteration $k$, the worst-ranked element has normalized rank $r_k = 1$,  hence $t_k = 1 - 1^\gamma = 0$, and sampling from this parent yields a draw  from $\nu$. The probability of selecting the worst-ranked element is
        \[p_k^{\mathrm{worst}} = \frac{e^{-\beta}}{\sum_{i=1}^{n_k} e^{-\beta r_k(x_i)}} \ge \frac{e^{-\beta}}{n_k}.\]
        Thus, for any measurable set $A \subseteq S$,
        \[\mathbb{P}(x_{k+1} \in A \mid \mathcal{A}_k) \ge \frac{e^{-\beta}}{n_k}\,\nu(A).\]
        Since $n_k = n_0 + k$, we have $\sum_{k=0}^\infty n_k^{-1} = \infty$, so any set $A$ with $\nu(A) > 0$ is visited infinitely often almost surely.

        For any $\delta > 0$, let $A_\delta = \{x \in S : f(x) > M - \delta\}$. By  assumption, $\nu(A_\delta) > 0$, so proposals enter $A_\delta$ infinitely often  almost surely, implying
        \[\liminf_{k\to\infty} \max_{0 \le j \le k} f(x_j) \ge M - \delta.\]
        Since $\delta > 0$ is arbitrary and $\max_{j \le k} f(x_j) \le M$ for all $k$,  the result follows.
    \end{proof}

\section{Thin, High-Dimensional Tube}
    \label{app:tube}
    For an even ambient dimension $D$, let $c:[0,1]\to\mathbb{R}^D$ be a smooth centerline defined by
    \[\begin{aligned}
        c_1(s) &= s,\\
        c_{2j-1}(s) &= a_j\sin(2\pi js+\phi_j),&& j=2,\ldots,D/2,\\
        c_{2j}(s) &= a_j\cos(2\pi js+\phi_j),&& j=1,\ldots,D/2.
    \end{aligned}\]
    where $a_j = 1/j$ and phases $\phi_j \sim \mathrm{Uniform}[0,2\pi]$ are sampled once and fixed across all methods and runs. Let
    \begin{equation}
        d(x,\mathcal{C}) = \min_{s\in[0,1]} \|x-c(s)\|_2,
        \qquad
        s^\star(x) \in \arg\min_{s\in[0,1]} \|x-c(s)\|_2.
    \end{equation}
    The objective rewards both proximity to the tube and progress along it:
    \begin{equation}
        f(x) = s^\star(x)\!\left(1 - \frac{d(x,\mathcal{C})}{r}\right),
    \end{equation}
    which is positive only inside the tube $d(x,\mathcal{C}) \leq r$, increases toward the endpoint $c(1)$, and penalizes deviation from the centerline continuously. Out-of-tube evaluations return a negative value; we clip reported rewards at zero.

    To train the sampler, we sample $D*1000$ arc lengths $s \sim \mathrm{Beta}(1,3)$, place points near $c(s)$ with isotropic Gaussian noise, and retain only samples inside the tube. The $\mathrm{Beta}(1,3)$ distribution concentrates mass near $s=0$. The prior captures the tube geometry without carrying any information about the objective direction.

\section{DiffAirfoil}
    \label{app:diffairfoil}
    \subsection{Architecture}
        DiffAirfoil~\citep{Wei24b} parameterizes airfoil geometry through an autodecoder originally introduced in~\citet{Wei23a}. Each airfoil is represented by a latent vector $z \in \mathbb{R}^{256}$ and the autodecoder $\mathcal{D}_\phi$ maps $z$ to a set of surface coordinates by learning a deformation field applied to a fixed template airfoil profile. Geometrically valid airfoils therefore correspond to a continuous deformation of the template, which enforces smoothness and prevents degenerate geometries by construction.

        The diffusion model originally proposed in~\citet{Wei24b} operates entirely in this latent space. In our implementation, we replace the diffusion sampler with a flow matching model (Appendix~\ref{app:exp_implem}), trained on the same latent codes. This substitution is inconsequential for SPARROW: the method requires only access to corruption and refinement operators $\mathcal{N}_t$ and $\mathcal{S}_{t \to 1}$, which flow matching provides directly. The autodecoder $\mathcal{D}_\phi$ is used solely for decoding proposals to airfoil coordinates before XFOIL evaluation, and is never updated.

    \subsection{Training Data}
        The autodecoder and generative prior are trained on the UIUC airfoil database~\citep{Selig96a}, supplemented with generated NACA profiles, yielding $N = 2{,}648$ geometrically valid airfoils (Table~\ref{tab:datasets}). No aerodynamic labels are used at any stage: the prior is purely geometric, encoding the distribution of plausible airfoil shapes without any information about lift, drag, or the objective direction.

    \subsection{Latent Space Structure}
        \begin{figure}[ht]
    \includegraphics[width=\textwidth]{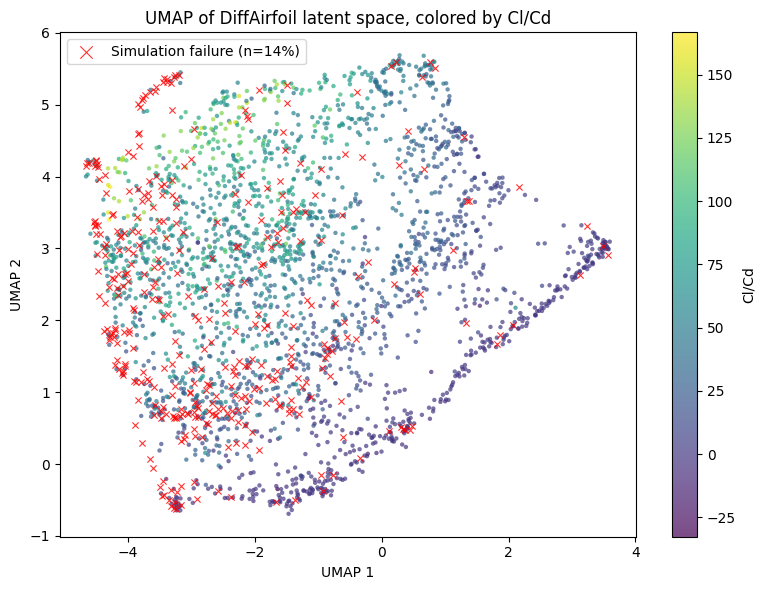}
    \caption{2-D UMAP of the foil training data in latent space. Colored by XFOIL CL/CD value.}
    \label{fig:umap}
\end{figure}
        The 256-dimensional latent space exhibits two key properties.
        \parag{Geometric structure without aerodynamic alignment.}
            Figure~\ref{fig:umap} shows a UMAP projection of the training latent codes, colored by $C_l/C_d$ evaluated by XFOIL. The distribution of Cl/Cd values is spatially structured: a compact region of the latent space concentrates high-performing airfoils, while the training data mass occupies a broader, aerodynamically mediocre region. The prior therefore encodes geometric validity across the full training distribution, but is not aligned with the aerodynamic objective.

        \parag{Uniform, objective-independent solver failures.}
            XFOIL is a panel-method solver whose convergence depends on flow conditions and airfoil geometry: turbulence, flow separation, and operation outside the solver's validity range cause frequent, unpredictable failures. Figure~\ref{fig:umap} shows the $C_l/C_d$ landscape exhibits a smooth objective gradient, while failures appear irregularly throughout the space. Crucially, solver failures are not concentrated near the optimum and do not reflect proximity to a geometric boundary, but constitute pure evaluation noise distributed across the latent space.

            Together, these properties define the challenge: the generative prior concentrates away from the optimum, and a significant number of evaluations are wasted regardless of where the search focuses. SPARROW addresses both simultaneously. Rank-based guidance steers proposals toward better-ranked regions. Rank-based invariance makes the optimization signal more robust to missing feedback, so that failed evaluations do not corrupt the search. Critically, the generative prior itself reduces the baseline failure rate by orienting the search towards neighborhoods of known valid airfoils. Classical methods operating directly in latent space more easily sample regions that exceed XFOIL's convergence domain.

\section{Design-Bench}
\label{app:designbench}
    Design-Bench~\citep{Trabucco22a} is a standard offline BBO benchmark: For each task, methods receive a fixed dataset of evaluated candidates and a surrogate oracle, a proxy trained on that dataset, to score proposals. We use it in an online fashion, querying the oracle sequentially as optimizations generate candidates. 
    \subsection{HopperController Landscape}
        \begin{figure}[ht]
    \includegraphics[width=\textwidth]{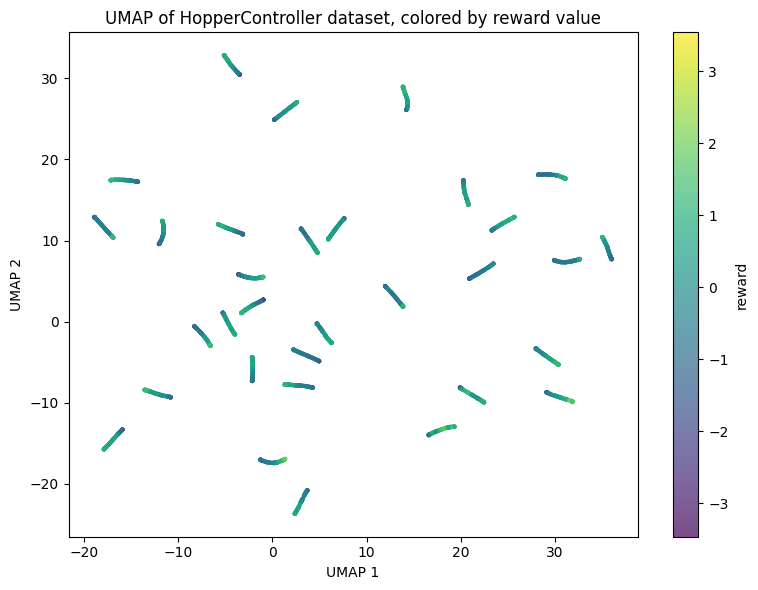}
    \caption{2-D UMAP of the HopperController dataset in search space. Colored by reward value.}
    \label{fig:umap_hopper}
\end{figure}
        The HopperController search space, shown on Figure~\ref{fig:umap_hopper} exhibits two key properties. First, it is significantly higher-dimensional than the other tasks and experiments, which makes traditional methods struggle. Second, the used oracle, trained on the available dataset, has multiple highly disconnected high-performing regions, that are relatively steep compared to the global structure. Search therefore has to balance local exploitation and global exploration.

    \subsection{Other Tasks}

\begin{figure}[ht]
    \centering
    \includegraphics[width=\textwidth]{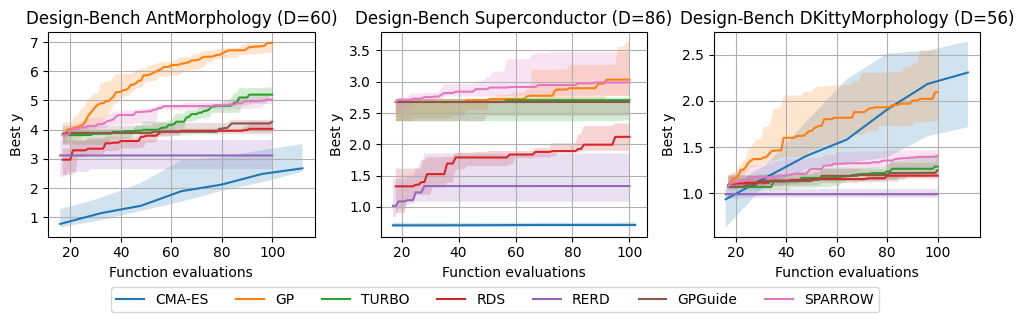}
    \caption{\small {\bf Objective median value and 25/75 Inter Quartile Range (IQR)} as a function of the number of times the objective function is evaluated. }
    \label{fig:db_plots}
\end{figure}

\begin{table}[ht]
    \centering
    \small
    \caption{\small Best objective value found (median, IQR over 10 runs)}
    \label{tab:db_tab}
    \begin{tabular}{lcccc}
        \toprule
        & \multicolumn{1}{c}{\textbf{AntMorphology}} & \multicolumn{1}{c}{\textbf{DKittyMorphology}} & \multicolumn{1}{c}{\textbf{SuperConductor}} \\
        \cmidrule(lr){2-2} \cmidrule(lr){3-3} \cmidrule(lr){4-4}
        \textbf{Method} & $B=100$ & $B=100$ & $B=100$ \\
        \midrule
        CMA-ES          & $2.67_{[2.55,\, 3.50]}$ & $\mathbf{2.30}_{[1.71,\, 2.64]}$ & $0.71_{[0.68,\, 0.75]}$\\
        GP              & $\mathbf{6.96}_{[6.63,\, 7.01]}$ & $2.09_{[1.77,\, 2.54]}$ & $\mathbf{3.03}_{[2.77,\, 3.64]}$\\
        TuRBO           & $5.19_{[5.15,\, 5.43]}$ & $1.28_{[1.20,\, 1.42]}$ & $2.70_{[2.37,\, 2.74]}$\\
        RDS             & $4.02_{[3.93,\, 4.27]}$ & $1.18_{[1.15,\, 1.23]}$ & $2.11_{[1.89,\, 2.33]}$\\
        \ours{} (ours)  & $5.01_{[4.38,\, 5.32]}$ & $1.40_{[1.28,\, 1.47]}$ & $2.98_{[2.77,\, 3.47]}$\\
        \bottomrule
    \end{tabular}
\end{table}
        
        As noted in \citet{Trabucco22a}, three of Design-Bench's tasks incorporate a key bias. The AntMorphology, DKittyMorphology, and HopperController datasets were collected using CMA-ES. This biases the results in two ways: First, the oracle is trained on CMA-ES trajectories, which densely cover regions reachable by ambient-space perturbations and are sparse elsewhere, rewarding ambient-space methods independently of solution quality. Second, our sampler is trained on the biased dataset, biasing its support towards CMA-ES friendly regions, reducing \ours{}'s structural advantage. This bias is severe at moderate dimensionality but negligible at $D=5126$, where CMA-ES itself collapses and its trajectories offer no meaningful coverage advantage. Superconductor uses real experimental data~\citep{Hamidieh18a} and is free of collection bias. However, its oracle retains a smoothness bias, with a comparatively low dimensionality $D=86$. 

        Results on AntMorphology, DKittyMorphology, and Superconductor, shown in Figure~\ref{fig:db_plots} and Table~\ref{tab:db_tab} reflect two compounding effects: the tasks' smoothness and relatively low complexity, combined with the benchmark bias, favor GP over \ours{}. Disentangling the two effects is difficult, but both make the tasks fall outside the regime \ours{} is built to solve. On HopperController in Section~\ref{sec:hopper} and the airfoil task in Section~\ref{sec:airfoil}, where the geometry is complex and the feedback is unreliable, \ours{} clearly outperforms competing approaches. Additionally, even on simpler smoother tasks, \ours{} outperforms the generative baselines, highlighting that this smoothness impacts all generative approaches and is not specific to \ours{}. This helps define the setting where generative approaches to BBO make sense: highly structured and complex spaces where traditional search struggles.

\section{Tuning Procedures}
    \label{app:tuning}
    \begin{table}[t]
    \centering
    \small
    \caption{Hyperparameter tuning on the 64D tube. Each configuration is evaluated once with
    $B=100$ using the same tuning seed. The configuration with the highest final best-so-far
    objective is selected. Selected configurations are shown in bold.}

    \label{tab:tuning_tube}
    \begin{tabular}{llr@{\hspace{1.5em}}llr}
        \toprule
        Method & Configuration & Best $f$
        & Method & Configuration & Best $f$ \\
        \midrule

        CMA-ES & $\sigma_0=1$    & -0.088
        & GP & $(K,m)=(100,0.1)$  & 0.075 \\
               & $\sigma_0=0.5$  & -0.277
               &    & $(K,m)=(50,0.1)$  & $0.075$ \\
               & $\sigma_0=0.25$ & -0.097
               &    & $(K,m)=(100,0.03)$ & 0.075 \\
               & $\sigma_0=0.1$  & -0.007
               &    & $(K,m)=(100,0.3)$  & 0.075 \\
               & $\sigma_0=\mathbf{0.05}$ & $\mathbf{0.013}$
               &    & $(K,m)=(20,0.1)$   & 0.075 \\

        \midrule

        TuRBO & $l=0.2$  & 0.075
        & RERD & $(L,N,K)=(3,10,4)$ & 0.138 \\
              & $l=0.4$  & 0.075
              &      & $(L,N,K)=\mathbf{(2,10,3)}$ & $\mathbf{0.165}$ \\
              & $l=0.8$  & 0.075
              &      & $(L,N,K)=(2,5,3)$  & 0.150 \\
              & $l=0.1$  & 0.075
              &      & $(L,N,K)=(3,5,2)$  & 0.139 \\
              & $l=0.05$ & $0.075$
              &      & $(L,N,K)=(3,1,4)$  & 0.085 \\

        \midrule

        GPGuide & $g=1$    & 0.088
        & SPARROW & $(\beta,\gamma,\lambda)=\mathbf{(50,1.0,0.1)}$ & $\mathbf{0.256}$ \\
                & $g=0.5$  & 0.088
                &         & $(\beta,\gamma,\lambda)=(100,2.0,0.05)$ & 0.153 \\
                & $g=0.25$ & 0.088
                &         & $(\beta,\gamma,\lambda)=(25,0.5,0.2)$   & 0.165 \\
                & $g=0.1$  & 0.178
                &         & $(\beta,\gamma,\lambda)=(75,1.5,0.1)$   & 0.235 \\
                & $g=\mathbf{0.03}$ & $\mathbf{0.186}$
                &         & $(\beta,\gamma,\lambda)=(40,0.8,0.15)$  & 0.218 \\

        \bottomrule
    \end{tabular}

\end{table}
    \begin{table}[t]
    \centering
    \small
    \caption{Hyperparameter tuning on HopperController. Each configuration is evaluated once with
    $B=100$ using the same tuning seed. The configuration with the highest final best-so-far
    objective is selected. Selected configurations are shown in bold.}

    \label{tab:tuning_hopper}
    \begin{tabular}{llr@{\hspace{1.5em}}llr}
        \toprule
        Method & Configuration & Best $f$
        & Method & Configuration & Best $f$ \\
        \midrule

        CMA-ES & $\sigma_0=1$    & -0.025
        & GP & $(K,m)=(100,0.1)$  & $0.787$ \\
               & $\sigma_0=0.5$  & -0.024
               &    & $(K,m)=(50,0.1)$  & 0.787 \\
               & $\sigma_0=0.25$ & 0.030
               &    & $(K,m)=(100,0.03)$ & 0.787 \\
               & $\sigma_0=0.1$  & 0.165
               &    & $(K,m)=(100,0.3)$  & 0.787 \\
               & $\sigma_0=\mathbf{0.05}$ & $\mathbf{0.189}$
               &    & $(K,m)=(20,0.1)$   & 0.787 \\

        \midrule

        TuRBO & $l=0.2$  & 0.787
        & RERD & $(L,N,K)=\mathbf{(3,10,4)}$ & $\mathbf{0.812}$ \\
              & $l=0.4$  & 0.787
              &      & $(L,N,K)=(2,10,3)$ & 0.812 \\
              & $l=0.8$  & 0.787
              &      & $(L,N,K)=(2,5,3)$  & 0.579 \\
              & $l=1.2$  & 0.787
              &      & $(L,N,K)=(3,5,2)$  & 0.579 \\
              & $l=1.6$ & 0.787
              &      & $(L,N,K)=(3,1,4)$  & 0.018 \\

        \midrule

        GPGuide & $g=1$    & 0.878
        & SPARROW & $(\beta,\gamma,\lambda)=(50,1.0,0.1)$ & 0.847 \\
                & $g=0.5$  & 0.878
                &         & $(\beta,\gamma,\lambda)=(10,1.0,0.1)$ & 0.854 \\
                & $g=0.25$ & 0.878
                &         & $(\beta,\gamma,\lambda)=(5,1.0,0.1)$   & 0.942 \\
                & $g=0.1$  & 0.878
                &         & $(\beta,\gamma,\lambda)=\mathbf{(5,0.1,0.1)}$ & $\mathbf{0.953}$ \\
                & $g=0.03$ & 0.878
                &         & $(\beta,\gamma,\lambda)=(5,0.1,0.01)$  & 0.944 \\

        \bottomrule
    \end{tabular}

\end{table}
    \begin{table}[t]
    \centering
    \small
    \caption{Hyperparameter tuning on the airfoil task. Each configuration is evaluated once with
    $B=100$ using the same tuning seed. The configuration with the highest final best-so-far
    objective is selected. Selected configurations are shown in bold.}

    \label{tab:tuning_foils}
    \begin{tabular}{llr@{\hspace{1.5em}}llr}
        \toprule
        Method & Configuration & Best $f$
        & Method & Configuration & Best $f$ \\
        \midrule

        CMA-ES & $\sigma_0=1$    & 144.526
        & GP & $(K,m)=\mathbf{(100,0.1)}$  & $\mathbf{189.475}$ \\
               & $\sigma_0=0.5$  & 102.925
               &    & $(K,m)=(50,0.1)$  & 176.812 \\
               & $\sigma_0=2$ & 142.276
               &    & $(K,m)=(100,0.03)$ & 177.273 \\
               & $\sigma_0=\mathbf{3}$  & $\mathbf{151.551}$
               &    & $(K,m)=(100,0.3)$  & 186.037 \\
               & $\sigma_0=5$ & 143.219
               &    & $(K,m)=(20,0.1)$   & 186.054 \\

        \midrule

        TuRBO & $l=\mathbf{0.2}$  & $\mathbf{144.343}$
        & RERD & $(L,N,K)=(3,10,4)$ & 99.790 \\
              & $l=0.4$  & 132.844
              &      & $(L,N,K)=\mathbf{(2,10,3)}$ & $\mathbf{104.559}$ \\
              & $l=0.8$  & 117.389
              &      & $(L,N,K)=(2,5,3)$  & 61.395 \\
              & $l=0.1$  & 134.741
              &      & $(L,N,K)=(3,5,2)$  & 63.305 \\
              & $l=0.05$ & 120.339
              &      & $(L,N,K)=(3,1,4)$  & 57.570 \\

        \midrule

        GPGuide & $g=\mathbf{1}$    & $\mathbf{146.122}$
        & SPARROW & $(\beta,\gamma,\lambda)=(50,1.0,0.1)$ & 119.115 \\
                & $g=0.5$  & 139.138
                &         & $(\beta,\gamma,\lambda)=(20,0.75,0.5)$ & 153.264 \\
                & $g=0.25$ & 131.913
                &         & $(\beta,\gamma,\lambda)=(10,0.5,1.0)$   & 170.440 \\
                & $g=0.1$  & 127.323
                &         & $(\beta,\gamma,\lambda)=\mathbf{(10,0.5,2.0)}$ & $\mathbf{175.284}$ \\
                & $g=0.03$ & 123.056
                &         & $(\beta,\gamma,\lambda)=(5,0.25,2.0)$  & 162.706 \\

        \bottomrule
    \end{tabular}

\end{table}
    \parag{The protocol: }
    Hyperparameter tuning also consumes objective evaluations and is therefore itself costly in the VLB setting. We consequently use a small and equal tuning allowance for all tunable methods rather than exhaustive hyperparameter optimization. For each task and method, we evaluate five hand-chosen candidate configurations, each for one full optimization run with budget $B=100$. The configuration is chosen considering the result of the previous tuning runs, the known structure of the search space and the known behaviour of the tunable hyperparameters. All candidates within a method--task pair use the same random seed. We select the configuration with the highest terminal best-so-far objective value. For the $64$D tube experiment at $B=1000$, we reuse the configuration selected at $B=100$ without further tuning. The selected configurations are then used unchanged for the runs reported in the main paper. Tables~\ref{tab:tuning_tube}, \ref{tab:tuning_hopper} and \ref{tab:tuning_foils} report all candidate configurations considered.

    \parag{The parameters: }
    For CMA-ES, $\sigma_0$ controls the initial search scale. In our tables it is expressed in normalized coordinates.
    For GP, $K$ caps the number of archive points used to fit the surrogate. When the archive exceeds $K$, we retain $80\%$ high-scoring points and $20\%$ randomly sampled points. The margin $m$ expands the acquisition-search domain beyond the coordinate-wise range of the evaluated archive. For TuRBO, $\ell$ sets the initial trust-region side length, which is subsequently adapted during optimization. For RERD, $N$ is the number of parallel samples, $L$ is the duplication factor used during reward-guided refinement, and $K$ is the number of refinement levels. Together they determine the number of objective calls in one complete refinement cycle, $Q=N((K-1)L+1)$. For GPGuide, $g$ scales the surrogate guidance applied along the generative trajectory. RDS has no tunable optimization hyperparameters.

    \parag{An analysis of the chosen configurations: }
    For HopperController, the latent space contains disconnected feasible regions, making this the hardest structural setting among our tasks. Very low $\beta=5$ encourages broad exploration across disconnected components rather than over-exploiting a single region. Very low $\gamma=0.1$ delegates most structural burden to the generative prior, whose corruption--refinement operators can bridge disconnected regions in ways that rank signals alone cannot. $\lambda=0.1$ was set by qualitative transfer from the tube experiment.

    For the airfoil experiment, lower $\beta=10$ reflects the importance of diversity in a multimodal aerodynamic landscape. Lower $\gamma=0.5$ increases reliance on the structural prior, which is useful in the failure-prone aerodynamic landscape. Larger $\lambda = 2$ accounts for the greater characteristic distances in the DiffAirfoil latent space.

\section{Raw Ablation Results}
    \label{app:ablation}
    \subsection{Tube dimensionality}
        \begin{table}[H]
    \centering
    \small
    \caption{\small
    Dimensionality ablation on the thin-tube benchmark at $B=100$.
    Values are median best objective with IQR over 10 runs.
    }
    \label{tab:dimension_ablation}
    \setlength{\tabcolsep}{3.5pt}
    \begin{tabular}{lccccc}
        \toprule
        \textbf{Method}
        & $\mathbf{D=8}$
        & $\mathbf{D=16}$
        & $\mathbf{D=32}$
        & $\mathbf{D=64}$
        & $\mathbf{D=128}$ \\
        \midrule

        CMA-ES
        & $0.10_{[.03,.35]}$
        & $0.17_{[.05,.28]}$
        & $0.01_{[.00,.04]}$
        & $0.00_{[.00,.00]}$
        & $0.00_{[.00,.00]}$ \\

        GP
        & $0.09_{[.08,.12]}$
        & $0.08_{[.07,.09]}$
        & $0.08_{[.05,.10]}$
        & $0.08_{[.07,.09]}$
        & $0.07_{[.07,.08]}$ \\

        TuRBO
        & $0.15_{[.11,.20]}$
        & $0.10_{[.09,.16]}$
        & $0.08_{[.06,.10]}$
        & $0.08_{[.07,.09]}$
        & $0.07_{[.07,.08]}$ \\

        RDS
        & $0.33_{[.31,.41]}$
        & $0.28_{[.24,.32]}$
        & $0.21_{[.18,.23]}$
        & $0.16_{[.14,.17]}$
        & $0.16_{[.14,.17]}$ \\

        RERD
        & $0.26_{[.24,.26]}$
        & $0.20_{[.18,.24]}$
        & $0.18_{[.17,.21]}$
        & $0.16_{[.13,.19]}$
        & $0.13_{[.12,.14]}$ \\

        GPGuide
        & $0.29_{[.26,.38]}$
        & $0.24_{[.21,.25]}$
        & $0.20_{[.16,.25]}$
        & $0.13_{[.09,.14]}$
        & $0.11_{[.10,.11]}$ \\

        \ours{} (ours)
        & $\mathbf{0.46}_{[.39,.54]}$
        & $\mathbf{0.32}_{[.28,.36]}$
        & $\mathbf{0.32}_{[.29,.35]}$
        & $\mathbf{0.30}_{[.28,.32]}$
        & $\mathbf{0.27}_{[.22,.29]}$ \\

        \bottomrule
    \end{tabular}
\end{table}
        \newpage
    \subsection{Hyperparameters}
        \begin{table}[H]
    \centering
    \small
    \caption{\small
    Hyperparameter sensitivity on the 64D thin-tube benchmark at $B=100$.
    Each parameter is varied independently while the other two are fixed at
    $\beta=50$, $\gamma=1$, and $\lambda=0.1$, shown in bold.
    Values report median best objective and IQR.
    }
    \label{tab:hyperparam_ablation}
    \setlength{\tabcolsep}{5pt}
    \begin{tabular}{cc@{\hspace{1.3em}}cc@{\hspace{1.3em}}cc}
        \toprule
        \multicolumn{2}{c}{$\beta$}
        & \multicolumn{2}{c}{$\gamma$}
        & \multicolumn{2}{c}{$\lambda$} \\
        \cmidrule(lr){1-2}
        \cmidrule(lr){3-4}
        \cmidrule(lr){5-6}
        Value & Best $f$
        & Value & Best $f$
        & Value & Best $f$ \\
        \midrule

        $10$  & $0.25_{[.19,.27]}$
        & $0.5$  & $0.19_{[.18,.23]}$
        & $0$    & $0.28_{[.24,.34]}$ \\

        $20$  & $0.26_{[.23,.27]}$
        & $0.75$ & $0.27_{[.23,.30]}$
        & $0.01$ & $0.28_{[.26,.32]}$ \\

        $\mathbf{50}$  & $\mathbf{0.30}_{[.28,.32]}$
        & $\mathbf{1}$ & $\mathbf{0.30}_{[.28,.32]}$
        & $0.03$ & $0.29_{[.22,.37]}$ \\

        $100$ & $0.25_{[.18,.34]}$
        & $1.25$ & $0.32_{[.27,.34]}$
        & $\mathbf{0.1}$ & $\mathbf{0.30}_{[.28,.32]}$ \\

        $200$ & $0.25_{[.23,.28]}$
        & $1.5$ & $0.24_{[.19,.27]}$
        & $0.3$ & $0.28_{[.25,.30]}$ \\

        & &
        & &
        $1$ & $0.25_{[.22,.29]}$ \\

        \bottomrule
    \end{tabular}
\end{table}

\section{Experimental Implementation}
    \label{app:exp_implem}
    \subsection{Generative Model Architecture}
\begin{figure}[ht]
\centering
\begin{tikzpicture}[
  font=\small,
  every node/.style={align=center},
  box/.style={
    draw, rounded corners=3pt, minimum width=2.2cm, minimum height=0.7cm,
    fill=gray!8, line width=0.4pt
  },
  resbox/.style={
    draw, rounded corners=3pt, minimum width=2.2cm, minimum height=0.7cm,
    fill=blue!6, line width=0.4pt
  },
  io/.style={
    draw, rounded corners=3pt, minimum width=2.2cm, minimum height=0.7cm,
    fill=orange!10, line width=0.4pt
  },
  tembox/.style={
    draw, rounded corners=3pt, minimum width=2.2cm, minimum height=0.7cm,
    fill=teal!8, line width=0.4pt, dashed
  },
  arr/.style={-{Stealth[length=4pt]}, line width=0.5pt},
  darr/.style={-{Stealth[length=4pt]}, line width=0.5pt, dashed},
  node distance=0.45cm
]

\node[io] (z)  {$z \in \mathbb{R}^d$};
\node[io, right=1cm of z] (t)  {$t \in [0,1]$};

\node[tembox, below=0.3cm of t] (tembed)
  {Sinusoidal\\embed ($\mathbb{R}^{64}$)};

\node[box, below=0.3cm of z] (proj)
  {Linear $\mathbb{R}^d \!\to\! \mathbb{R}^d$};

\node[resbox, below=0.3cm of proj] (b1)
  {Residual block};
\node[resbox, below=0.3cm of b1]  (b2)
  {Residual block};
\node[resbox, below=0.3cm of b2]  (b3)
  {Residual block};
\node[resbox, below=0.3cm of b3]  (b4)
  {Residual block};

\node[box, below=0.3cm of b4] (out)
  {Linear $\mathbb{R}^d \!\to\! \mathbb{R}^d$};

\node[io, below=0.3cm of out] (v)
  {$v_\theta(z,t) \in \mathbb{R}^d$};

\draw[arr] (z)    -- (proj);
\draw[arr] (proj) -- (b1);
\draw[arr] (b1)   -- (b2);
\draw[arr] (b2)   -- (b3);
\draw[arr] (b3)   -- (b4);
\draw[arr] (b4)   -- (out);
\draw[arr] (out)  -- (v);

\draw[arr] (t) -- (tembed);

\foreach \blk in {b1,b2,b3,b4}{
  \draw[darr] (tembed.west) -| ([xshift=0.5cm]\blk.east) -- (\blk.east);
}

\begin{scope}[xshift=7cm, yshift=-2cm]

  \node[font=\footnotesize\itshape] (zoom) at (0, 1.2) {Residual block (detail)};

  \node[box] (ln)  at (0, 0)   {LayerNorm};
  \node[box, below=0.3cm of ln]  (cat) {Cat$[\,h,\,\tau(t)\,]$};
  \node[box, below=0.3cm of cat] (l1)  {Linear + SiLU};
  \node[box, below=0.3cm of l1]  (l2)  {Linear};
  \node[circle, draw, minimum size=0.5cm, below=0.38cm of l2,
        fill=gray!8, line width=0.4pt, font=\small] (add) {$+$};

  \draw[arr] (ln)  -- (cat);
  \draw[arr] (cat) -- (l1);
  \draw[arr] (l1)  -- (l2);
  \draw[arr] (l2)  -- (add);

  \coordinate (top) at ([yshift=0.35cm]ln.north);
  \coordinate (bot) at ([yshift=-0.1cm]add.south);
  \draw[arr, rounded corners=4pt](top) -- +(1.5,0) |- (add.east);

  \draw[arr] ([yshift=0.55cm]ln.north) -- (ln.north);
  \draw[arr] (add.south) -- +(0,-0.45);

  \draw[darr] ([xshift=-1cm]cat.west) -- (cat.west)
  node[midway, above, font=\scriptsize] {$\tau(t)$};
\end{scope}

\end{tikzpicture}
\caption{%
  Flow matching model architecture, shared across all experiments. The ambient dimension $d$ varies per task (See Appendix~\ref{app:training_data}). The noise level $t$ is encoded by a 64-dimensional sinusoidal embedding $\tau(t)$ and injected into each residual block by concatenation with the hidden state. Each block applies LayerNorm, two linear layers with SiLU activation, and a residual connection. The network maps $(z,t)\in\mathbb{R}^d\times[0,1]$ to a velocity $v_\theta(z,t)\in\mathbb{R}^d$.
}
\label{fig:timedresnet}
\end{figure}
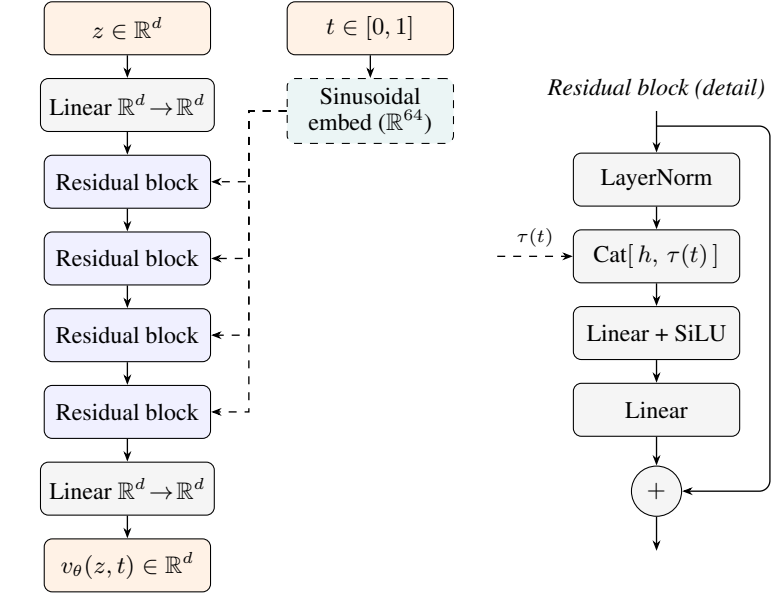
        All generative priors share a common architecture (Figure~\ref{fig:timedresnet}): a four-block residual network with sinusoidal time conditioning, trained with a flow matching objective. The hidden dimension equals the ambient dimension $d$ in all experiments.
    \subsection{Flow Matching Training}
        \begin{algorithm}[ht]
\caption{Flow Matching Training Step}
\label{alg:flow-training}
\begin{algorithmic}[1]
\Require Training set $\mathcal{D}$, model $v_\theta$, batch size $B$
\State Sample minibatch $\{z_1^{(i)}\}_{i=1}^B \sim \mathcal{D}$
\State Sample noise $z_0^{(i)} \sim \mathcal{N}(\mu_\mathcal{D},\, \sigma_\mathcal{D}^2 I)$ for each $i$
\State Reorder $\{z_0^{(i)}\}$ to minimise $\sum_i \|z_0^{(i)} - z_1^{(i)}\|^2$ \Comment{minibatch OT}
\State Sample $t^{(i)} \sim \mathrm{Beta}(0.5, 0.5)$ for each $i$
\State $z_t^{(i)} \leftarrow (1 - t^{(i)})\,z_0^{(i)} + t^{(i)}\,z_1^{(i)}$
\State $v^{\star(i)} \leftarrow z_1^{(i)} - z_0^{(i)}$
\State Update $\theta$ on $\mathcal{L} = \frac{1}{B}\sum_i \|v_\theta(z_t^{(i)},\, t^{(i)}) - v^{\star(i)}\|^2$
\end{algorithmic}
\end{algorithm}
        All models are trained with a linear interpolation bridge using the procedure described in Algorithm~\ref{alg:flow-training}. Noise samples are drawn from $\mathcal{N}(\mu_\mathcal{D}, \sigma_\mathcal{D}^2 I)$, fitted from the training dataset, and are reordered via minibatch optimal transport (Hungarian algorithm on pairwise squared $\ell_2$ cost) to reduce training variance without altering marginal distributions. Time is sampled from $\mathrm{Beta}(0.5, 0.5)$, which upweights the endpoints of the trajectory.

        All models are trained for 100'001 steps with Adam (lr $= 10^{-4}$), cosine annealing with warm restarts ($T_0 = 10'000$, $T_\mathrm{mult} = 2$, $\eta_\mathrm{min} = 10^{-7}$), and mixed-precision training. Batch size is $512$.

        At sampling time, initial noise is drawn from $\mathcal{N}(\mu_\mathcal{D}, \sigma_\mathcal{D}^2 I)$, consistent with training. Integration uses RK4 with $\max(5,\lfloor(1 - t_\mathrm{start}) \times 100\rfloor)$ steps.
    \subsection{Training Datasets}
        \label{app:training_data}
        \begin{table}[ht]
\centering
\caption{Training set sizes and dimensionality for each generative prior.}
\label{tab:datasets}
\begin{tabular}{llll}
\toprule
Task & Source & $D$ & $N$ \\
\midrule
Thin tube & Synthetic & $D$ & $1'000 \times D$ \\
Airfoil & UIUC database + NACA profiles & $256$ & $2'648$ \\
AntMorphology & Design-Bench & $60$ & $10'004$ \\
DKittyMorphology & Design-Bench & $56$ & $10'004$ \\
HopperController & Design-Bench & $5126$ & $3'200$ \\
Superconductor & Design-Bench & $86$ & $17'014$ \\
\bottomrule
\end{tabular}
\end{table}
        Dataset sizes for each task are given in Table~\ref{tab:datasets}. The thin tube prior is explained in Section~\ref{sec:toy_tube} and Appendix~\ref{app:tube}. The airfoil prior is described in Appendix~\ref{app:diffairfoil}. All Design-Bench priors are trained directly on the provided training sets.
    \subsection{Baseline Implementations}
        \paragraph{CMA-ES.} We use the pycma \citep{Hansen16} library with default parameters. $x_0$ is initialized with a random training data point. $\sigma_0$ is chosen by the tuning.

        \paragraph{GP.} We use the BoTorch \citep{Balandat20a} implementation of vanilla Bayesian optimisation with a Matérn-5/2 kernel and Expected Improvement acquisition, following \citet{Hvarfner24a} with their dimension-scaled length-scale priors.

        \paragraph{TuRBO.} We implement TuRBO following the official BoTorch tutorial, with default parameters and the same archive cap of $K=200$ points (80\% top-ranked, 20\% random).

        In all methods, failed evaluations (e.g.\ solver crashes in the airfoil task) are excluded from the archives/populations but counted against the budget $B$.
        
        \paragraph{RERD.} We follow \citet{Uehara25a} with the hyperparameters and budget accounting specified in Appendix~\ref{app:tuning}.

        \parag{GPGuide.} We guide the generative trajectory using the gradient of a GP's expected improvement, evaluated at an estimate of the final candidate and added to the sampler's velocity with bounded strength $g$.
    
    \subsection{XFOIL evaluation settings.}
        All aerodynamic evaluations use XFOIL~\citep{Drela89} at a fixed angle of attack $\alpha = 0^\circ$ and Reynolds number $Re = 10^6$, applied during both the latent space analysis of Appendix~\ref{app:diffairfoil} and the optimization experiments of Section~\ref{sec:airfoil}.

    \subsection{Compute}
        All experiments were run on a single consumer GPU (NVIDIA RTX 4060, 8GB). Generative model training took at most a few hours per task. Optimization runs completed in at most a few hours. Consistent with our low-budget framing, we consider generative model inference as negligible relative to the cost of objective evaluation: the dominant expense in all experiments is querying $f$, not running the sampler.
    
    \subsection{Licenses}
        Design-Bench~\citep{Trabucco22a} and BoTorch~\citep{Balandat20a} are MIT licensed. pycma~\citep{Hansen25a} is BSD licensed. XFOIL~\citep{Drela89} is distributed under the GNU GPL. The airfoil coordinates of the UIUC Airfoil Database~\citep{Selig96a} are publicly available with no stated license. All are used in accordance with their respective licenses.

\end{document}